\documentclass[twoside]{article}
\usepackage[accepted]{for_arxiv}
\usepackage{amssymb}
\usepackage{amsmath}
\usepackage{amsthm}
\usepackage{comment}
\usepackage{graphicx}
\usepackage{textcomp}
\usepackage{url}
\usepackage[ruled,linesnumbered,vlined]{algorithm2e}
\newcommand{\textcolor}[2]{#2}

\newtheorem{prop}{Proposition}
\newtheorem{thm}{Theorem}
\newtheorem{lem}{Lemma}
\usepackage{natbib}
\setcitestyle{authoryear,open={(},close={)}}
\makeatletter
\newcommand{\removelatexerror}{\let\@latex@error\@gobble}
\newcommand{\ratioT}[3]{T_{#2,#3}(#1)}
\DeclareMathOperator*{\argmax}{\arg\,\max}
\DeclareMathOperator*{\argmin}{\arg\,\min}
\makeatother
\begin{document}
\runningauthor{Liyuan Xu, Junya Honda, Masashi Sugiyama}
\twocolumn[

\aistatstitle{Fully adaptive algorithm for pure exploration in linear bandits}

\aistatsauthor{ Liyuan Xu${}^{\dagger\ddagger}$ \And Junya Honda${}^{\dagger\ddagger}$ \And Masashi Sugiyama${}^{\ddagger\dagger}$ }

\aistatsaddress{ $\dagger$\,:The University of Tokyo \And $\ddagger$\,:RIKEN }
]

\begin{abstract}
We propose the first fully-adaptive algorithm for pure exploration in linear bandits---the task to find the arm with the largest expected reward, which depends on an unknown parameter linearly. While existing methods partially or entirely fix sequences of arm selections before observing rewards, 
our method adaptively changes the arm selection strategy based on past observations at each round. We show our sample complexity matches the achievable lower bound up to a constant factor in an extreme case. Furthermore, we evaluate the performance of the methods by simulations based on both synthetic setting and real-world data, in which our method shows vast improvement over existing ones.
\end{abstract}

\section{Introduction}
The \emph{multi-armed bandit} (MAB) problem \citep{Robbins1952} is a sequential decision-making problem, where the agent sequentially chooses
one arm out of $K$ arms
and receives a stochastic reward drawn from a fixed, unknown distribution related with the arm chosen. While most of the literature on the MAB focused on the maximization of the cumulative rewards, we consider the \emph{pure-exploration} setting or the \emph{best arm identification} problem \citep{bubeck2009pure}. Here, the goal of the agent is to identify the arm with the maximum expected reward. 

The best arm identification has recently gained increasing attention, and a considerable amount of work covers many variants of it. For example, \citet{Audibert2010} considered fixed budget setting, where the agent tries to minimize the  misspecification probability in a fixed number of trials, and \citet{Even-dar2003} introduced fixed confidence setting, where the agent tries to minimize the number of trials until the probability of misspecification becomes smaller than a fixed threshold. 

An important extension of the MAB is the \emph{linear bandit} (LB) problem \citep{Auer2003}. In the LB problem, each arm has its own feature $x\in\mathbb{R}^d$, and the expected reward can be written as $x^\top\theta$, where $\theta\in\mathbb{R}^d$ is an unknown parameter and $x^\top$ is the transpose of $x$. Although there are a number of studies on the LB \citep{Abbasi-Yadkori2010, Li2010}, most of them aim for maximization of the cumulative rewards, and only a few consider the pure-exploration setting.



In spite of the scarce literature, the best arm identification problem on LB has a wide range of applications. For example, \citet{Hoffman2014} applied the pure exploration in LB to the optimization of a traffic sensor network and automatic hyper-parameter tuning in machine learning. Furthermore, even if the goal of the agent is to maximize the cumulative rewards, such as the case of news recommendation \citep{Li2010}, considering pure exploration setting is sometimes helpful when the system cannot respond feedback in real-time after once launched.

The first work that addressed
the LB best arm identification problem was by \citet{Hoffman2014}. They studied the best arm identification in the fixed-budget setting with correlated reward distributions and devised an algorithm called BayesGap, which is a Bayesian version of a gap based exploration algorithm \citep{Gabillon2012}.

Although BayesGap outperformed algorithms
that ignore the correlation,
there is a drawback that it never pulls arms turned out to be sub-optimal,
which can significantly harm the performance in LB.
For example, consider the case where there are three arms and the feature of them are $x_1=(1,0)^\top,\,x_2=(1,0.01)^\top$, and $x_3=(0,1)^\top$, respectively. Now, if $\theta = (\theta_1,\theta_2)^\top = (2,0.01)^\top$,
then the expected reward of arms 1 and 2
are close to each other, hence it is hard to figure out the best arm just by observing the samples from them.
On the other hand,
pulling arm 3 greatly reduces the samples required, since it enhances the accuracy of estimation of $\theta_2$. As illustrated in this example, selecting a sub-optimal arm can give valuable insight for
comparing near-optimal arms in LB.

\citet{Soare2014} is the first work
taking this nature into consideration. They studied the fixed-confidence setting and derived an algorithm based on transductive experimental design \citep{Yu2006}, called $\mathcal{XY}$-static allocation.
The algorithm employs a static arm selection strategy, in the sense that it
fixes all arm selections before observing any reward.
Therefore, it is not able to focus on estimating near-optimal arms, thus the algorithm can only be the worst case optimal. 

In order to develop more efficient algorithms,
it is necessary to
pull arms adaptively based on past observations so that most samples are allocated for comparison of near-optimal arms.
The difficulty in constructing an adaptive strategy is that
a confidence bound for statically selected arms is not always applicable
when arms are adaptively selected.
In particular, a confidence bound for an adaptive strategy introduced by \citet{Abbasi-Yadkori2010} is looser than a bound for a static strategy derived from Azuma's inequality \citep{azuma1967} by a factor of $\sqrt{d}$ in some cases, where $d$ is the dimension of the feature.
\citet{Soare2014} tried to mitigate this problem by introducing a semi-adaptive algorithm called $\mathcal{XY}$-adaptive allocation, which divides rounds into multiple phases and
uses different static allocations in different phases.
Although this theoretically improves the sample complexity,
the algorithm has to
discard all samples collected in previous phases to make the confidence bound for static strategies applicable, which drops the empirical performance significantly.

To discuss tightness of the sample complexity of
$\mathcal{XY}$-adaptive allocation,
\cite{Soare2014} introduced
the $\mathcal{XY}$-oracle allocation algorithm,
which assumes access to the true parameter $\theta$ for selecting arms to pull.
They discussed that the sample complexity of this algorithm
can be used as a lower bound on the sample complexity for this problem
and claimed that
the upper bound on the sample complexity of $\mathcal{XY}$-adaptive allocation is close to this lower bound.
However, the derived upper bound is not given in an explicit form and
contains a complicated term coming from $\mathcal{XY}$-static allocation
used as a subroutine. 
In fact,
the sample complexity of $\mathcal{XY}$-adaptive allocation
is much worse than that of $\mathcal{XY}$-oracle allocation, as we will see numerically in Section~\ref{sec:experiment-synthetic}.

Our contribution is to develop a novel fully adaptive algorithm,
which changes arm selection strategies based
on all of the past observations at every round.
Although this prohibits us from using a tighter bound for static strategies, we show that the factor $\sqrt{d}$ can be avoided by the careful construction of the confidence bound,
and the sample complexity almost matches that of $\mathcal{XY}$-oracle allocation.
We conduct experiments to evaluate the performance of
the proposed algorithm, showing that it requires ten times less samples than existing methods to achieve the same level of accuracy.

\section{Problem formulation}
We consider the LB problem, where there are $K$ arms with features $x_1,\dots,x_K \in \mathbb{R}^d$. We denote the set of the features as $\mathcal{X} = \{x_1,\dots,x_K\}$ and the largest $l_2$-norm of the features as $L=\max_{i\in \{1,\dots,K\}}\|x_i\|_2$. At every round $t$, the agent selects an arm $a_t \in [K] = \{1,\dots,K\}$, and observes immediate reward $r_t$, which is characterized by 

\[r_t = x_{a_t}^\top\theta + \varepsilon_t.\]
Here, $\theta \in \mathbb{R}^d$ is an unknown parameter, and $\varepsilon_t$ represents a noise variable, whose expectation equals zero. We assume that the $l_2$-norm of $\theta$ is less than $S$ and the noise distribution is conditionally $R$-sub-Gaussian, which means that noise variable $\varepsilon_t$ satisfies 
\begin{align*}
\mathbb{E}\left[\mathrm{e}^{\lambda\varepsilon_t}|x_{a_1},\dots,x_{a_{t-1}},\varepsilon_1,\dots,\varepsilon_{t-1}\right] \leq \exp\left(\frac{\lambda^2R^2}{2}\right)
\end{align*}
for all $\lambda\in\mathbb{R}$. This condition requires the noise distribution to have zero expectation and $R^2$ or less variance \citep{Abbasi-Yadkori2010}.
As prior work \citep{Abbasi-Yadkori2010,Soare2014}, we assume that
parameters $R$ and $S$ are known to the agent. 

We focus on the $(\varepsilon,\delta)$-best arm identification problem. Let $a^*=\argmax_i x_i^\top\theta $ be the best arm, and $x^*$ be the feature of arm $a^*$. The problem is to design an algorithm to find arm $\hat{a}^*$ which satisfies
\begin{align}
\mathbb{P}[(x^*-x_{\hat{a}^*})^\top\theta \geq \varepsilon] \leq \delta,\label{eq:stop}
\end{align}
as fast as possible.

\section{Confidence Bounds}\label{sec:confidence_bound}
In order to solve the best arm identification in the LB setting, the agent sequentially
estimates $\theta$ from past observations and bounds the estimation error. However, 
if arms are selected adaptively based on past observations, 
the estimation becomes much more complicated
compared to the case where pulled arms are fixed in advance.
In this section, we discuss this difference and
how we can construct a tight bound for an algorithm
with an adaptive selection strategy.

Given the sequence of arm selections ${\bf x}_n = (x_{a_1},\dots,x_{a_n})$,
one of the most standard estimators for $\theta$ is the least-square estimator
given by
\begin{align}
\hat{\theta}_n = A_{{\bf x}_n }^{-1}b_{{\bf x}_n },\nonumber
\end{align}
where $A_n$ and $b_n$ are defined as
\[
A_{{\bf x}_n } = \sum_{t=1}^n x_{a_t}x^\top_{a_t},\quad b_{{\bf x}_n } = \sum_{t=1}^n x_{a_t}r_t .
\]
\citet{Soare2014} used the ordinary least-square estimator $\hat\theta_n$ combined with the following 
proposition on the confidence ellipsoid for $\hat{\theta}_n$,
which is derived from Azuma's inequality \citep{azuma1967}.
\begin{prop}[{\citealp[Proposition 1]{Soare2014}}]\label{prop_static}
Let noise variable $\varepsilon_t$ be bounded as $\varepsilon \in [-\sigma,\sigma]$ for $\sigma > 0$,
then, for any fixed sequence ${{\bf x}_n }$, statement
\begin{align}
|x^\top\theta-x^\top\hat\theta_n| \geq 2\sigma \|x\|_{A^{-1}_{{\bf x}_n }}\sqrt{2\log\left(\frac{6n^2K}{\delta\pi^2}\right)}\label{eq:bound_static}
\end{align}
holds for all $n\in\mathbb{N}$ and $x\in\mathcal{X}$ with probability at least $1-\delta$. 
\label{prop:linear-confidence-azuma}
\end{prop}


Here, the matrix norm $\|x\|_A$ is defined as  $\|x\|_A = \sqrt{x^\top Ax}$. The assumption that $\mathbf{x}_n$ is fixed is essential in Prop.~\ref{prop_static}.
In fact, if $\mathbf{x}_n$ is adaptively determined depending on past observations,
then the estimator $\hat{\theta}_n$ is no more unbiased and it becomes essential to consider the
regularized least-squares estimator $\hat{\theta}_n^\lambda$ given by
\begin{align}
\hat{\theta}^\lambda_n = (A_{{\bf x}_n }^\lambda)^{-1}b_{{\bf x}_n },\nonumber
\end{align}
where $A_{{\bf x}_n }^\lambda$ is defined by
\[
A_{{\bf x}_n }^\lambda = \lambda I + \sum_{t=1}^n x_{a_t}x^\top_{a_t},
\]
for $\lambda > 0$ and the identity matrix $I$.
For this estimator, we can use another confidence bound
which is valid even if an adaptive strategy is used.

\begin{prop}[{\citealp[Theorem 2]{Abbasi-Yadkori2010}}] In the LB with conditionally $R$-sub-Gaussian noise, if the $l_2$-norm of parameter $\theta$ is less than $S$ and the arm selection only depends on past observations, then statement
\begin{align*}
 |x^\top(\hat{\theta}^\lambda_n- \theta)|\leq\|x\|_{(A_{{\bf x}_n }^\lambda)^{-1}}C_n
 \end{align*}
holds
for given $x \in \mathbb{R}^d$ and all $n>0$
with probability at least $1-\delta$, where $C_n$ is defined as 
\begin{align}
C_n = R\sqrt{2\log\frac{\mathrm{det}(A_{{\bf x}_n }^\lambda)^{\frac12}\mathrm{det}(\lambda I)^{-\frac12}}{\delta}} + \lambda^{\frac12}S. \label{eq:confidence-ellipsoid}
\end{align}
Moreover, if $\|x_{a_t}\|\leq L$ holds for all $t>0$, then
\begin{align}
 C_n \leq R\sqrt{d\log\frac{1+nL^2/\lambda}{\delta}} + \lambda^{\frac12}S. \label{eq:confidence-ellipsoid-bound}
\end{align}
\label{prop:linear-confidence}
\end{prop}
Although the bound in \eqref{eq:confidence-ellipsoid-bound}
holds regardless of whether the arm selection strategy is static or adaptive,
the bound is looser than Prop.~\ref{prop:linear-confidence-azuma} by an extra factor $\sqrt{d}$
when a static strategy is considered.

In the following sections, we use the bound in \eqref{eq:confidence-ellipsoid} to construct an algorithm
that adaptively selects arms based on past data.
We reveal that the extra factor $\sqrt{d}$ arises
from looseness of \eqref{eq:confidence-ellipsoid-bound} and
the sample complexity
can be bounded without this factor by an appropriate evaluation
of \eqref{eq:confidence-ellipsoid}.


\section{Arm Selection Strategies}\label{sec:Optimal mixture method}
In order to minimize the number of samples, the agent has to select arms that reduce the interval of the confidence bound as fast as possible. In this section, we discuss such an arm selection strategy, and in particular, we consider the strategy to reduce the matrix norm $\|x_i-x_j\|_{A_{{\bf x}_n}^{-1}}$, which represents the uncertainty in the estimation of the gap of expected rewards between arms $i$ and $j$.

\citet{Soare2014} introduced the strategy called $\mathcal{XY}$-static allocation, which makes the sequence of selection ${\bf x}_n$ to be
\begin{align}
\argmin_{{\bf x}_n}\max_{y\in\mathcal{Y}} \|y\|_{A^{-1}_{{\bf x}_n }}.\label{eq:xy-static-allocation}
\end{align}
In \eqref{eq:xy-static-allocation}, $\mathcal{Y}$ is the set of directions defined as $\mathcal{Y} = \{x-x'|x,x'\in\mathcal{X}\}$. The problem is to minimize the confidence bound of the direction hardest to estimate, which is known as transductive experimental design \citep{Yu2006}. Note that this problem does not depend on the past reward, which satisfies the prerequisite of Prop.~\ref{prop:linear-confidence-azuma}. 

A drawback of this strategy is that it treats all directions $y\in\mathcal{Y}$ equally.
Since our goal is to find the best arm $a^*$, we are not interested in estimating the gaps between all arms but the gaps between the best arm and the rest. Therefore, we should focus on the directions in $\mathcal{Y}^* = \{x^*-x|x\in\mathcal{X}\}$, where $x^*$ is the feature of the best arm. Furthermore, directions in $\mathcal{Y}^*$ are still not equally important, since we need more samples to distinguish the arms whose expected reward is close to that of the best arm.

In order to overcome this weakness while using Prop.~\ref{prop:linear-confidence-azuma}, \citet{Soare2014} proposed a semi-adaptive strategy called the $\mathcal{XY}$-adaptive strategy. This strategy partitions rounds into multiple phases
and arms to select are static within a phase but changes between phases. At the beginning of phase $j$, it constructs a set of potentially optimal arms $\hat{\mathcal{X}}_j$ based on the samples collected during the previous phase $j-1$. Then, it selects the sequence ${\bf x}_n$ in phase $j$ as
\begin{align}
\argmin_{{\bf x}_n}\max_{y\in\hat{\mathcal{Y}}_j} \|y\|_{A^{-1}_{{\bf x}_n }},\label{eq:xy-adaptive-allocation}
\end{align}
for $\hat{\mathcal{Y}}_j = \{x-x'|x,x\in\hat{\mathcal{X}}_j\}$.
As it goes through the phases, the size of $\hat{\mathcal{X}}_j$ decreases so that the algorithm can focus on discriminating a small number of arms.

Although the $\mathcal{XY}$-adaptive strategy can avoid the extra factor $\sqrt{d}$ in \eqref{eq:confidence-ellipsoid-bound}, the agent has to reset the design matrix $A_{{\bf x}_n}$ at the beginning of each phase in order to make Prop.~\ref{prop:linear-confidence-azuma} applicable. As experimentally shown in Section~\ref{sec:experiment}, we observe that this empirically degenerates the performance considerably.

On the other hand, our approach is fully adaptive, which selects arms based
on all of the past observations
at every round. More specifically, at every round $t$,
the algorithm chooses (but not pulls) a pair of arms,
$i_t$ and $j_t$, the gap of which needs to be estimated.
Then, it selects an arm so that the sequence of selected arms becomes close to
\begin{align}
{\bf x}^*_n(i_t,j_t) = \argmin_{{\bf x}_n} \|y(i_t,j_t)\|_{(A^\lambda_{{\bf x}_n })^{-1}},\label{eq:lingape-selection-strategy}
\end{align}
where $y(i_t,j_t) = x_{i_t}-x_{j_t}$. Although Prop.~\ref{prop:linear-confidence-azuma} is no longer applicable in our strategy, it can focus on
the estimation of the gap between the best arm and near-optimal arms.

\section{LinGapE Algorithm}
In this section, we present a novel algorithm for $(\varepsilon,\delta)$-best arm identification in LB. We name the algorithm \emph{LinGapE (Linear Gap-based Exploration)}, as it is inspired by UGapE \citep{Gabillon2012}.

The entire algorithm is shown in Algorithm~\ref{LinGapE}.
At each round, LinGapE first chooses two arms, the arm with the largest estimated rewards $i_t$ and the most ambiguous arm $j_t$.
Then, it pulls the most informative arm to estimate the gap of expected rewards $(x_{i_t} -x_{j_t})^\top\theta$ by Line~\ref{line:arm-selection} in Algorithm~\ref{LinGapE}.

An algorithm of choosing arms $i_t$ and $j_t$ is presented in Algorithm~\ref{Select-Direction}, where we denote the estimated gap by $\hat{\Delta}_t(i,j) = (x_i - x_j)^\top \hat{\theta}^\lambda_t$ and the confidence interval of the estimation by $\beta_t(i,j)$ defined as 
\begin{align}
\beta_t(i,j) = \|y(i,j)\|_{A_t^{-1}}C_t \label{eq:confidence-interval},
\end{align} 
for $C_t$ given in \eqref{eq:confidence-ellipsoid}.

\begin{algorithm}[t]
\SetKwInOut{Input}{Input}\SetKwInOut{Output}{Output}
\SetKw{KwInit}{Initialize}
\SetKwFor{Loop}{Loop}{}{EndLoop}
\label{LinGapE}
\Input{accuracy $\varepsilon$, confidence level $\delta$, norm of unknown parameter $S$, noise level $R$}
\Output{the arm $\hat{a}^*$ which satisfies stopping condition \eqref{eq:stop}}
\BlankLine
Set $A_0 \leftarrow \lambda I,\, b_0 \leftarrow {\bf 0},\, t\leftarrow0$\;
\tcp{Initialize by pulling each arm once}
\For{$i \in [K]$}{
    $t \leftarrow t+1$\;
    Observe $r_t \leftarrow  x^\top_i\theta + \varepsilon_t$\;
    Update $A_t$ and $b_t$\;
    $T_i(t) \leftarrow 1$\;
}
\Loop{}{
    \tcp{Select which gap to examine}
    Select-direction($t$)\;
    \If{$B(t) \leq \varepsilon$}{
         \Return $i_t$ as the best arm $\hat{a}^*$\;
    }
    \tcp{Pull the arm based on the gap}
    Select the arm $a_{t+1}$ based on \eqref{eq:arm-selection-1} or \eqref{eq:arm-selection-2} \;\label{line:arm-selection}
    $t\leftarrow t+1$\;
    Observe $r_t \leftarrow x^\top_{a_t}\theta + \varepsilon_t$\;
    Update $A_t$ and $b_t$\;
    $T_{a_t}(t) \leftarrow T_{a_t}(t) + 1$\;
}
\caption{LinGapE}
\end{algorithm}

\begin{algorithm}[t]
\SetKwProg{Def}{Procedure}{:}{end}
\SetKwFunction{SelectDirection}{Select-direction}
\label{Select-Direction}
\Def{\SelectDirection{$t$}}
{$\hat{\theta}^\lambda_t \leftarrow A^{-1}_tb_t$\;
    $i_t \leftarrow \argmax_{i\in[K]} (x_i^\top\hat{\theta}^\lambda_t)$\label{line:i-def}\; 
    $j_t \leftarrow \argmax_{j\in[K]} (\hat{\Delta}_t(j,i_t) + \beta_t(j,i_t))$\label{line:j-def}\;
    $B(t) \leftarrow \max_{j\in[K]}  (\hat{\Delta}_t(j,i_t) + \beta_t(j,i_t))$\;\label{line:B-def}
    }

\caption{Select-direction}
\end{algorithm}

\subsection{Arm Selection Strategy}
After choosing arms $i_t$ and $j_t$,
the algorithm has to select arm $a_t$,
which most decreases the confidence bound $\beta_t(i_t,j_t)$, or equivalently, $\|y(i,j)\|_{A_t}$.
As in \citet{Soare2014}, we propose two procedures for this.

One is to select arms greedily, which is
\begin{align}
\!a_{t+1} = \argmin_{a \in [K]} (y(i_t,j_t))^\top(A_t + x_{a}x^\top_{a})^{-1}y(i_t,j_t). \label{eq:arm-selection-1}
\end{align}
We were not able to gain a theoretical guarantee of the performance for this greedy strategy, though our experiment shows that it performs well.

The other is to consider the optimal selection ratio of each arm for decreasing $\|y(i_t,j_t)\|_{A^{-1}_t}$. Let $p^*_i(y(i_t,j_t))$ be the ratio of arm $i$ appearing in the sequence ${\bf x}^*_n(y(i,j))$ in \eqref{eq:lingape-selection-strategy} when $n\to\infty$. By the discussion given in Appendix~\ref{sec:proof-p}, we have
\begin{align}
p_i^*(y(i_t,j_t)) = \frac{|w_i^*(y(i_t,j_t))|}{\sum_{i=1}^K |w_i^*(y(i_t,j_t))|} ,\label{eq:p-def}
\end{align}
where $w^*_i(y(i_t,j_t))$ is the solution of the linear program
\begin{align}
\argmin_{\{w_i\}} \sum_{i=1}^K |w_i| \quad \mathrm{s.t.}\,y(i_t,j_t) = \sum_{i=1}^K w_ix_i . \label{eq:w-problem}
\end{align}
The optimization is easier compared with \citet{Soare2014}, who solved \eqref{eq:xy-static-allocation} and \eqref{eq:xy-adaptive-allocation} via nonlinear convex optimization.

\textcolor{red}{We pull the arm that makes the ratio of arm selections close to ratio $p^*_i(y)$.
To be more precise, $a_{t+1}$ is decided by 
\begin{align}
a_{t+1} = \argmin_{a \in [K]:\,p^*_a(y(i_t,j_t))>0} T_a(t)/p^*_a(y(i_t,j_t)), \label{eq:arm-selection-2}
\end{align}
where $T_a(t)$ is the number of times that arm $a$ is pulled until $t$-th round.
This strategy is a little more complicated than the greedy strategy in \eqref{eq:arm-selection-1} but enjoys a simple theoretical characteristic, based on which we conduct analysis.}

LinGapE is capable of solving $(\varepsilon,\delta)$-best arm identification, regardless of which strategy is employed, as stated in the following theorem.
\begin{thm}
Whichever the strategy in \eqref{eq:arm-selection-1} or \eqref{eq:arm-selection-2} is employed, arm $\hat{a}^*$ returned by LinGapE satisfies condition \eqref{eq:stop}.\label{thm:justification}
\end{thm}
The proof can be found in Appendix~\ref{sec:proof}.\label{sec:Algorithm-selection-strategy}
\subsection{Comparison of Confidence Bounds}
A distinctive character of LinGapE is that it considers an upper confidence bound of the \emph{gap of rewards}, while UGapE and other algorithms for LB, such as OUFL \citep{Abbasi-Yadkori2010}, consider an upper confidence bound of the \emph{reward of each arm}. This approach is, however, not suited for the pure exploration in LB, where the gap plays an essential role.

The following example illustrates the importance of considering such quantities. Consider that there are three arms, features of which are $x_1 = (-10,10)^\top,\,x_2 = (-9,10)^\top,$ and $x_3=(-1,0)^\top$.
Assuming that we have
$\hat{\theta}^\lambda_t = (\hat{\theta}^\lambda_{t,(1)},\hat{\theta}^\lambda_{t,(2)})^{\top}=(-1,0)^\top$,
thus the estimated best arm is $i_t = 1$.
Now, let us consider
the case where we have already been confident that $\hat{\theta}^\lambda_{t,(1)}\approx -1$
but still unsure of $\hat{\theta}^\lambda_{t,(2)}\approx 0$.
In such a case, algorithms considering an upper confidence bound of the rewards of each arm, such as UGapE, choose arm $2$ as $j_t$, since it has a larger estimated expected reward and a longer confidence interval than arm $3$. However, it is not efficient, since arm 2 cannot have a larger expected reward than arm 1 when $\theta_1 = 1$. On the other hand, LinGapE can avoid this problem, since a confidence interval for $(x_1-x_3)^\top\hat{\theta}^\lambda_t$ is longer than $(x_1-x_2)^\top\hat{\theta}^\lambda_t$.

\section{Sample Complexity}
In this section, we give an upper bound of the sample complexity of LinGapE and compare it with existing methods.
\subsection{Sample Complexity}
Here, we bound the sample complexity of LinGapE when arms to pull are selected by \eqref{eq:arm-selection-2}. Let the problem complexity $H_\varepsilon$ be defined as 
\begin{align}
H_\varepsilon = \sum_{k=1}^K \max_{i,j\in[K]}\frac{p^*_k(y(i,j))\rho(y(i,j))}{\max\left(\varepsilon,\frac{\varepsilon+\Delta_i}{3},\frac{\varepsilon+\Delta_j}{3}\right)^2},
\label{eq:problem-complexity}
\end{align}
where $\Delta_i$ is defined as
\begin{align}
\Delta_i = \begin{cases}
(x_{a^*}-x_i)^\top\theta &(i \neq a^*),\\
\argmin_{j\in[K]} (x_{a^*}-x_j)^\top\theta &(i = a^*),
\end{cases}\label{eq:delta-def}
\end{align}
and $\rho(y(i,j))$ is the optimal value of problem \eqref{eq:w-problem}, denoted as 
\begin{align}
\rho(y(i,j)) = \sum_{k=1}^K |w^*(y(i,j))|.\label{eq:rho-def}
\end{align}
Then, the sample complexity of LinGapE can be bounded as follows.
\begin{thm}\label{thm:sample-complexity}
Assume that $a_t$ is determined by \eqref{eq:arm-selection-2}.
If $\lambda \leq \frac{2R^2}{S^2}\log\frac{K^2}{\delta}$, then
the stopping time $\tau$ of LinGapE satisfies
\begin{align}
\mathbb{P}\left[\tau\leq 8H_\varepsilon R^2\log\frac{K^2}{\delta} +
C(H_\varepsilon,\delta)\right] \geq 1-\delta\label{eq:sample-complexity},
\end{align}
where $C(H_\varepsilon,\delta)$ is specified in \eqref{eq:detail-C} of Appendix~\ref{sec:proof-comp-sample} and satisfies
\[C(H_\varepsilon,\delta)  =\mathcal{O}\left(dH_\varepsilon\log\left(dH_\varepsilon \log\frac{1}{\delta}\right)\right).\]
Furthermore, if $\lambda > 4H_\varepsilon R^2L^2$, then
\begin{align}
\mathbb{P}\left[\tau \leq \left(8H_\varepsilon R^2\log\frac{K^2}{\delta}  + 4H_\varepsilon\lambda S^2+2K\right)\right]\geq 1-\delta. \label{eq:sample-complexity-high-dim}
\end{align}
\end{thm}


The proof can be found in Appendix~\ref{sec:proof}. The theorem states that there are two types of sample complexity. The first bound \eqref{eq:sample-complexity} is more practically applicable, since the condition $\lambda \leq \frac{2R^2}{S^2}\log\frac{K^2}{\delta}$ can be checked by known parameters. On the other hand, we cannot ensure the condition $\lambda > 4H_\varepsilon R^2L^2$ is satisfied, since we cannot know $H_\varepsilon$ in advance. However, the second bound in \eqref{eq:sample-complexity-high-dim} can be tighter than the first one in \eqref{eq:sample-complexity} if there are only few directions that $\theta$ needs to be estimated accurately, where we have $H_\varepsilon \ll d$. In such a case, the additional term $4H_\varepsilon\lambda S^2+2K$ is much smaller than $C(H_\varepsilon,\delta)=\mathcal{O}\left(dH_\varepsilon\log\left(dH_\varepsilon \log\frac{1}{\delta}\right)\right)$, since $4H_\varepsilon\lambda S^2+2K = \mathcal{O}(H^2_\varepsilon)$ when $\lambda \simeq 4H_\varepsilon R^2L^2$.

\subsection{Discussion on Problem Complexity}

The problem complexity \eqref{eq:problem-complexity} has an interesting relation with that of the $\mathcal{XY}$-oracle allocation algorithm introduced by \citet{Soare2014}.
They considered the case where the agent knows true parameter $\theta$ when selecting an arm to pull,
and tries to \emph{confirm arm $a^*$ is actually the best arm}. Then, an efficient strategy is to let the sequence of arm selections {${\bf x}_n$} be
\begin{align}
 \argmin_{{\bf x}_n}\max_{i \in [K]\backslash \{a^*\}}\frac{\|y(a^*,i)\|_{A^{-1}_{{\bf x}_n}}}{\Delta_i}. \label{eq:xy-oracle-allocation}
\end{align}
An upper bound of the sample complexity of $\mathcal{XY}$-oracle allocation is proved to be $\mathcal{O}(H_{\text{oracle}}\log(1/\delta))$, where problem complexity $H_{\text{oracle}}$ is defined as
\[H_{\text{oracle}} = \max_{i \in [K]\backslash \{a^*\}}\frac{\rho(y(a^*,i))}{\Delta^2_i}.\]
This is expected to be close to the achievable lower bound of the problem complexity \citep{Soare2014}. Here, we prove a theorem that points out the relation between $H_\text{oracle}$ and our problem complexity $H_\varepsilon$.

\begin{thm}\label{thm:problem-complexity-bound}
Let $H_0$ be the problem complexity of LinGapE \eqref{eq:problem-complexity} when $\varepsilon$ is set as $\varepsilon=0$. Then, we have
\[H_0 \leq 72H'_{\mathrm{oracle}},\]
where $H'_{\mathrm{oracle}}$ is defined as
\[
H'_{\mathrm{oracle}}
= \sum_{i \in [K]\backslash \{a^*\}}\frac{\rho(y(a^*,i))}{\Delta^2_i}.\]
\end{thm}
The proof of the theorem can be found in Appendix~\ref{sec:proof-prob-complexity}. Since $H'_{\text{oracle}} \leq KH_{\text{oracle}}$, this result shows that our problem complexity matches the lower bound up to a factor of $K$, the number of arms.
Furthermore,
if $\Delta_i$ for some $i$ is very small compared with
$\{\Delta_{i'}\}_{i'\neq i}$, that is, if
there is only one near-optimal arm, then
$H_{\text{oracle}}$ becomes close to $H'_{\text{oracle}}$, and hence our problem complexity $H_0$ achieves the lower bound up to a constant factor.

\citet{Soare2014} claimed that $\mathcal{XY}$-adaptive allocation achieves this lower bound as well. To be precise, they discussed that the sample complexity of $\mathcal{XY}$-adaptive allocation is $\mathcal{O}(\max(M^*,N^*))$, where $N^*$ is the sample complexity of $\mathcal{XY}$-oracle allocation. Nevertheless, they did not give an explicit bound of $M^*$, which stems from the static strategy employed in each phase. Our experiments in Section~\ref{sec:experiment} show that $M^*$ can be as large as the sample complexity of $\mathcal{XY}$-static allocation,
the problem complexity of which is proved to be $\Omega(4d/\Delta^2_{a^*})$ and can be arbitrarily larger than $H_\mathrm{oracle}$ in the case of $d\to\infty$ \citep{Soare2014}.
Therefore, LinGapE is the first algorithm that always achieves the lower bound up to a factor of $K$.

We point out another interpretation of our problem complexity. If the features set $\mathcal{X}$ equals the set of canonical bases $(e_1,e_2,\dots,e_d)$, then the LB problem reduces to the ordinary MAB problem. In such a case, $p^*_k(y(i,j))$ and $\rho(y(i,j))$ are computed as 
\begin{align*}
p^*_k(y(i,j)) &= \begin{cases}
\frac12 &(k=i \mbox{ or } k=j),\\
0 &(\text{otherwise}),
\end{cases}\\
\rho(y(i,j)) &= 4.
\end{align*}
Therefore, if the noise variable is bounded in the interval $[-1,1]$, which is known as $1$-sub-Gaussian, the problem complexity becomes
\[H_\varepsilon = \sum_{k=1}^K \frac{2}{\max\left(\varepsilon,\frac{\varepsilon+\Delta_i}{3}\right)^2} \leq \frac{9}{8}H_\varepsilon^{\mathrm{UGapE}},\]
where $H_\varepsilon^{{\mathrm{UGapE}}}$ is the problem complexity of UGapE \citep{Gabillon2012}.  This fact suggests that 
LinGapE incorporates the linear structure into UGapE from the perspective of the problem complexity.

\section{Experiments}\label{sec:experiment}
In this section, we compare the performance of LinGapE with the algorithms proposed by \cite{Soare2014} through experiments in two synthetic settings and simulations based on real data.

\subsection{Experiment on Synthetic Data}\label{sec:experiment-synthetic}
We conduct experiments in two synthetic settings. One is the setting where an adaptive strategy is suitable, and the other is where pulling all arms uniformly becomes the optimal strategy. We set the noise distribution as $\varepsilon_t \sim \mathcal{N}(0,1)$ and run LinGapE with parameters of $\lambda = 1,\, \varepsilon = 0$ and $\delta = 0.05$ in both cases. We observed that altering the arm selection strategy in \eqref{eq:arm-selection-1} and in \eqref{eq:arm-selection-2} has very little impact on the performance,
and we plot the results only for the greedy strategy \eqref{eq:arm-selection-1}. We repeated experiments ten times for each experimental setting, the average of which is reported.

\subsubsection{Setting where an Adaptive Strategy is Suitable}
\label{sec:experiment-synthetic-1}
The first experiment is conducted on the setting where the adaptive strategy is favored, which is introduced by \citet{Soare2014}. We set up the LB problem with $d+1$ arms, where features consist of canonical bases $x_1 = e_1,\,\dots,\,x_d= e_d$ and an additional feature $x_{d+1} = (\cos(0.01),\sin(0.01),0,\dots,0)^\top$. The true parameter is set as $\theta = (2,0,\dots,0)^\top$ so that the expected reward of arm $d+1$ is very close to that of the best arm $a^*=1$ compared with other arms. Hence, the performance heavily depends on how much the agent can focus on comparing arms 1 and $d+1$.

Figure~\ref{fig:experiment} is a semi-log plot of the average stopping time of LinGapE, in comparison with the $\mathcal{XY}$-static allocation, $\mathcal{XY}$-adaptive allocation and $\mathcal{XY}$-oracle allocation
algorithms,
all of which are introduced by \citet{Soare2014}. The arm selection strategies of them are given in \eqref{eq:xy-static-allocation}, \eqref{eq:xy-adaptive-allocation} and \eqref{eq:xy-oracle-allocation}, respectively. The result indicates the superiority of LinGapE to the existing algorithms.
\begin{figure}[t]
\centerline{\includegraphics[width=\linewidth]{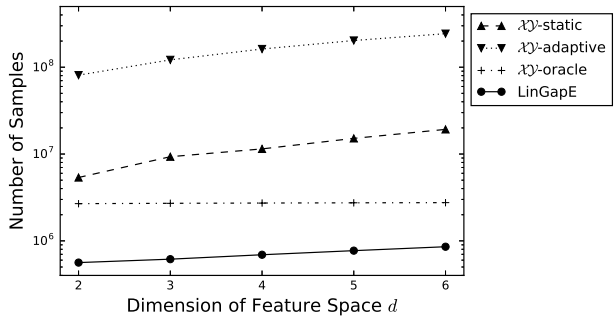}}
\caption{The number of samples required to estimate the best arm in the synthetic setting introduced by \citet{Soare2014}. }
\label{fig:experiment}
\end{figure}

This difference is due to the adaptive nature of LinGapE. While $\mathcal{XY}$-static allocation treats all directions $y\in\mathcal{Y}$ equally, LinGapE is able to identify the most important direction $y$ from the past observations and select arms based on it. Table~\ref{table:arm-selection}, which shows the number of times that each arm is pulled when $d=5$, supports this idea.
From this table, we can see that while $\mathcal{XY}$-static allocation pulls all arms equally, LinGapE and $\mathcal{XY}$-oracle allocation pull arm $2$ frequently. This is an efficient strategy for estimating the gap of expected rewards between arms $1$ and $d+1$, since the feature of arm $2$ is almost aligned with the direction of $x_1-x_{d+1}$. We can conclude that LinGapE is able to focus on discriminating arms $1$ and $d+1$, which reduces the total number of required samples significantly. 

\begin{table}[tb]
\caption{An example of arm selection when $d = 5$.}
    \begin{center}
    \begin{tabular}{|c|c|c|c|} \hline
      &$\mathcal{XY}$-static&LinGapE&$\mathcal{XY}$-oracle\\\hline
      Arm $1$&1298590&2133&13646\\\hline
      Arm $2$&2546604&428889&2728606\\\hline
      Arm $3$&2546666&19&68\\\hline
      Arm $4$&2546666&34&68\\\hline
      Arm $5$&2546666&33&68\\\hline
      Arm $6$&1273742&11&1\\\hline
    \end{tabular}
    \end{center}
    \label{table:arm-selection}
\end{table}

Although $\mathcal{XY}$-adaptive allocation has adaptive nature as well, it performs much worse than $\mathcal{XY}$-static allocation
in this setting.
This is due to the limitation that it has to reset the design matrix $A_{{\bf x}_n}$ at every phase. We observe that it actually succeeds to find $\hat{\mathcal{X}}_j = \{1,d+1\}$ in the first few phases. However, it
``forgets''
the discarded arms and gets $\hat{\mathcal{X}}_{j+1}=\mathcal{X}$ again. This is because the agent pulls only arms $1$, $2$ and $d+1$ at phase $j$ for estimating $(y(1,d+1))^\top\theta$, and the design matrix $A_{{\bf x}_n}$ constructed at the phase $j$ cannot discard other arms anymore. Therefore, the algorithm still handles all $\mathcal{X}$ at the last phase, which requires as many samples as $\mathcal{XY}$-static allocation. Hence, this is an example that the sample complexity of $\mathcal{XY}$-adaptive allocation matches that of $\mathcal{XY}$-static allocation. 
We observed that the same happened in the subsequent two experiments and $\mathcal{XY}$-adaptive performed at least five times worse than $\mathcal{XY}$-static allocation. Therefore, in order to highlight the difference between $\mathcal{XY}$-static allocation and LinGapE in linear scale plots, we do not plot the result for $\mathcal{XY}$-adaptive allocation in the following.

It is somewhat surprising
that LinGapE wins over $\mathcal{XY}$-oracle allocation, given that it
assumes
access to the true parameter $\theta$.
The main reason for this is that our confidence bound is tighter than that used in $\mathcal{XY}$-oracle allocation. This seems contradicting, since our confidence bound $\beta_t(i,j)$ is looser by a factor of $\sqrt{d}$ in the worst case where $\det(A_t) = \mathcal{O}(t^d)$ as discussed in Section~\ref{sec:confidence_bound}. Nevertheless, $\det(A_t)$ grows almost linearly with $t$ in our setting, since LinGapE mostly pulls the same arm as presented in  Table~\ref{table:arm-selection}, which significantly reduces the length of the confidence interval. This suggests the sample complexity given in Theorem~\ref{thm:sample-complexity} is actually loose, in which we bound $\det(A_t)$ by $\mathcal{O}(t^d)$ as well (see Prop.~\ref{prop:bound-of-confidence-interval} in Appendix~\ref{sec:proof-comp-sample}).

\subsubsection{Setting where a Static Strategy is Optimal} \label{sec:experiment-synthetic-2}
We conduct another experiment in synthetic setting, where $\mathcal{XY}$-static allocation is almost optimal. We consider the LB with $K=d=5$, where the feature set $\mathcal{X}$ equals the canonical set $(e_1,e_2,\dots,e_{5})$. We set the parameter $\theta$ as $\theta = (\Delta,0,\dots,0)^\top$, 
where $\Delta>0$, hence arm $1$ has a larger expected reward by $\Delta$ than all other arms. As $\Delta \to 0$, we need to estimate all arms equally accurately, therefore the optimal strategy is to pull all arms uniformly, which corresponds to $\mathcal{XY}$-static allocation. 

The result for various gaps $\Delta$ is shown in Figure~\ref{fig:experiment3}. We observe not only that LinGapE performs better than $\mathcal{XY}$-static allocation but also that the gap of the performance increases as $\Delta\to0$, where $\mathcal{XY}$-static allocation can be thought as the optimal strategy. A reason for this is that while $\mathcal{XY}$-static allocation always pulls arms uniformly until all arms satisfy the stopping condition, LinGapE quits pulling arms that is once turned out to be sub-optimal, which prevents LinGapE from observing unnecessary samples. This enhances the performance, especially in the case of $\Delta\to0$, where the number of samples needed for discriminating each arm is severely influenced by the noise.

\begin{figure}[t]
\centerline{\includegraphics[width=0.9\linewidth]{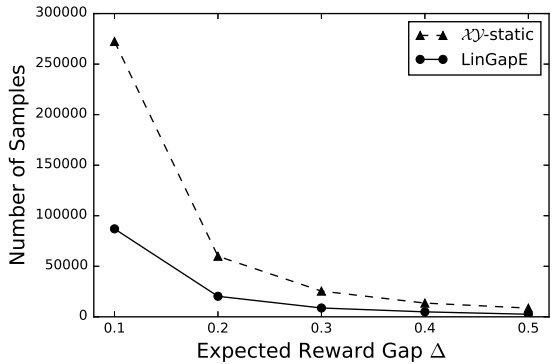}}
\caption{The number of samples required to estimate the best arm in the synthetic setting where the arm selection strategy in  $\mathcal{XY}$-static allocation is an almost optimal strategy. }
\label{fig:experiment3}
\end{figure}

\subsection{Simulation Based on Real Data} \label{sec:experiment-realistic}
We conduct another experiment based on a real-world dataset. We use Yahoo! Webscope Dataset R6A\footnote{\url{https://webscope.sandbox.yahoo.com/}},
which consists of
features of $36$-dimensions accompanied with binary outcomes. It is originally used as an unbiased evaluation benchmark for the LB aiming for cumulative reward maximization \citep{Li2010}, and we slightly change the situation so that it can be adopted for pure exploration setting. We construct the 36 dimensional feature set $\mathcal{X}$ by the random sampling from the dataset, and the reward is generated by
\[r_t = \begin{cases}
1 & \left(\text{w.p. }(1+x_{a_t}^\top\theta^*)/{2}\right),\\
-1 & (\text{otherwise}),
\end{cases}\]
where $\theta^*$ is the regularized least squared estimator
fitted
for the original dataset. Although $x_{a_t}^\top\theta^*$ is not necessarily bounded in $[-1,1]$, we observe that $x^\top\theta^*\in[-1,1]$ for all features $x$ in the dataset. Therefore, $(1+x_{a_t}^\top\theta^*)/2$ is always a valid probability
in this case.
We compare the performance with the $\mathcal{XY}$-static allocation algorithm, where the estimation is given by the regularized least squared estimator with $\lambda = 0.01$. The detailed procedure can be found in Appendix~\ref{sec:procedure}.

The average number of samples required in ten simulations is shown in Figure~\ref{fig:experiment2}, in which LinGapE performs roughly five times better than the $\mathcal{XY}$-static strategy, and the gap of performances increases as we consider more arms. This is because the $\mathcal{XY}$-static strategy tries to estimate all arms equally, while LinGapE is able to focus on estimating the best arm even if there are many arms.

\begin{figure}[t]
\centerline{\includegraphics[width=0.9\linewidth]{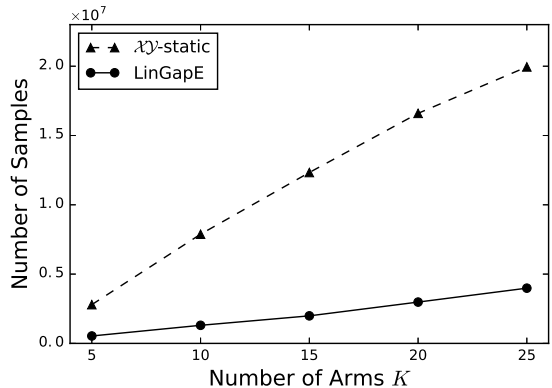}}
\caption{The number of samples required to estimate the best arm on Yahoo! Webscope Dataset R6A. }
\label{fig:experiment2}
\end{figure}

\section{Conclusions}
In this paper, we studied the pure exploration in the linear bandits.
We first reviewed a drawback in the existing work, and then introduced a novel fully adaptive algorithm, LinGapE.
We proved that the sample complexity of LinGapE matches the lower bound in an extreme case
and
confirmed
its superior performance in the experiments.
Since LinGapE is the first algorithm that achieves this lower bound, we would like to consider its various extensions and develop computationally efficient algorithms. In particular, pure exploration in the fixed budget setting is a promising direction of extension, since LinGapE is shares many ideas with UGapE, which is known to be applicable in the fixed budget setting as well \citep{Gabillon2012}.
Furthermore, as explained
in Section~\ref{sec:experiment-synthetic},
the derived sample complexity may be improved
since the evaluation of the determinant in Prop.~\ref{prop:bound-of-confidence-interval} given in Appendix~\ref{sec:proof-comp-sample} is loose.
The bound based on the tight evaluation of the determinant remains for the future work.

\subsubsection*{Acknowledgements}
LX utilized the facility provided by Masason Foundation. JH acknowledges support by KAKENHI 16H00881, and MS acknowledges support by KAKENHI 17H00757.

\subsubsection*{Reference}
\bibliography{references}
\newpage
\appendix
\section{Detailed Procedure of Simulation Based on Real-World Data}\label{sec:procedure}
In this appendix
we give the detailed procedure of the experiment presented in Section~\ref{sec:experiment-realistic}.  We use the Yahoo! Webscope dataset R6A, which consists of more than 45 million user visits to the Yahoo! Today module collected over 10 days in May 2009. The log describes the interaction (view/click) of each user with one randomly chosen article out of 271 articles. It was originally used as an unbiased evaluation benchmark for the LB in explore-exploration setting \citep{Li2010}. The dataset is made of features describes each user $u$ and each article $a$, both are expressed in 6 dimension feature vectors, accompanied with a binary outcome (clicked/not clicked).  We use article-user interaction feature $z_{a,u} \in \mathbb{R}^{36}$, which is expressed by a Kronecker product of a feature vector of article $a$ and that of $u$. \citet{chu2009} present a detailed description of the dataset, features and the collection methodology.

In our setting, we use the subset of the dataset which is collected on the one day
(May 1st).
We first conduct the regularized linear regression
on whether the target is clicked ($r_t=1$) or not clicked ($r_t=-1$).
Here, the regularize term is set as $0.01$.
Let $\theta^*$
be the learned parameter,
which we regard
as the
``true''
parameter in the simulation. We consider the LB with $K$ arms, the features of which are sampled from the dataset. We limit the the case of $\Delta_i \geq 0.05$ for all arms $i$ in order to make the problem not too hard. The reward $r_t$ at the $t$-th round is given by 
\[r_t = \begin{cases}
1 & \left(\text{w.p. }\frac{1+x_{a_t}^\top\theta^*}{2}\right)\\
-1 & (\text{otherwise})
\end{cases},\]
where $x_{a_t}$ is the feature of the arm selected at the $t$th round. Although it does not always the case, $x^\top\theta^*$ is happended to be bounded in $[-1,1]$ for all feature $x$ in the dataset, therefore $(1+x_{a_t}^\top\theta^*)/2$ is always valid for probability. Furthermore, since $x_{a_t}^\top\theta^* \in [-1,1]$, the noise variable
$\varepsilon_t$ is bounded as $\varepsilon_t \in [-2,2]$,
which is known as $2$-sub-Gaussian. We run LinGapE on this setting, where the parameter is fixed as $\varepsilon = 0$, $\delta = 0.05$,
and $\lambda = 1$,
in comparison with $\mathcal{XY}$-static allocation, where the estimation is given by regularized least squared estimator with $\lambda = 0.01$.

\section{Derivation of Ratio $p^*(y)$ }\label{sec:proof-p}
In this appendix, we present the derivation of $p_k^*(y(i,j))$ defined in \eqref{eq:p-def} and the proof of Lemma~\ref{lem:bound-of-matrix-norm}, which bounds the matrix norm when the arm selection strategy based on the ratio $p^*_k(y(i,j))$. 

The original problem of reducing the interval of confidence bound for given $y\in\mathcal{Y}$ is to obtain
\[\argmin_{{\bf x}_n} \|y\|_{(A^\lambda_{{\bf x}_n })^{-1}}\]
in the limit of $n\to\infty$.
Since we choose features from the finite set $\mathcal{X}$ in the LB, the problem becomes
\begin{align}
\min_{C_i \in \mathbb{N}\cup\{0\}} y^\top\left(\lambda I + \sum_{i=1}^K C_ix_ix^\top_i\right)^{-1}y \quad \mathrm{s.t.} \sum_{i=1}^K C_i = n. \label{eq:appendB-tmp1}
\end{align}
where the $C_i$ represents the number of times that the arm $i\in[K]$ is pulled before the $n$-th round.

We first conduct the continuous relaxation, which turns the optimization problem \eqref{eq:appendB-tmp1} into
\[\min_{p_i \geq 0} \frac1n y^\top\left(\frac{\lambda}{n} I + \sum_{i=1}^K p_ix_ix^\top_i\right)^{-1}y \quad \mathrm{s.t.} \sum_{i=1}^K p_i = 1,\]
where $p_i$ corresponds to the ratio $C_i /n$. Although this relaxed problem can be solved by convex optimization, it is not suited for the LB setting because the solution depends on the sample size $n$. Therefore, we consider the asymptotic case, where the sample size $n$ goes to infinity. 

It is proved \citep[Thm 3.2]{Yu2006} that the continuous relaxed problem is equivalent to 
\begin{align}
&\min_{p_i,w_i} \left\|y - \sum_{i=1}^K w_ix_i\right\|^2 + \frac{\lambda}{n} \sum_{i=1} \frac{w_i^2}{p_i} \notag\\
&\mathrm{s.t.} \sum_{i=1}^K p_i = 1,\, p_i \geq 0,\, p_i,w_i \in \mathbb{R}. \label{eq:2}
\end{align}
Since we consider $y\in\mathcal{Y}$, there always exists $w_i$ such that $y = \sum_{i=1}^K w_ix_i$.
Then, $\{w_i\}$ such that $\|y - \sum_{i=1}^K w_ix_i\| > 0$ cannot be the optimal solution
for sufficiently small $\lambda/n$ and thus
the optimal solution has to satisfy $\|y - \sum_{i=1}^K w_ix_i\| = 0$.
Therefore, the asymptotic case of \eqref{eq:2} corresponds to the problem
\begin{align}
\min_{p_i,w_i}\, & \sum_{i=1} \frac{w_i^2}{p_i} \notag\\
\mathrm{s.t.}\, &y = \sum_{i=1}^K w_ix_i\notag\\
&\sum_{i=1}^K p_i = 1,\, p_i\geq 0,\, w_i \in \mathbb{R},\label{eq:3}
\end{align}
the KKT condition of which yields the definition of $p^*$ in \eqref{eq:p-def}. 

If we employ the arm selection strategy in \eqref{eq:arm-selection-2} based on $p^*$ in \eqref{eq:p-def}, we can bound the matrix norm $\|y(i,j)\|_{A_t^{\-1}}$ as described in the following lemma. 
\begin{lem}\label{lem:bound-of-matrix-norm}
Recall that $\rho(y(i,j))$ and $p^*_k(y(i,j))$ are defined in \eqref{eq:rho-def} and \eqref{eq:p-def},
respectively. Let $T_i(t)$ be the number of times that the arm $i$ is pulled before the $t$-th round.
Then, the matrix norm $\|y(i,j)\|_{A_t^{-1}}$ is bounded by
\[\|y(i,j)\|_{A_t^{-1}} \leq \sqrt{\frac{\rho(y(i,j))}{\ratioT{t}{i}{j}}},\]
where
\[\ratioT{t}{i}{j} = \min_{\substack{k \in [K]:\\ p^*_k(y(i,j))>0}} T_k(t)/p^*_k(y(i,j)). \]
\end{lem}

This lemma is proved by the following lemma.
\begin{lem}\label{lem:matrix-inequality}
Let $A$ be
a positive definite matrix
in $\mathbb{R}^{d\times d}$ and $x,y$ be vectors in $\mathbb{R}^{d}$.
Then,
for any constant $\alpha > 0$, 
\[y^\top(A+\alpha xx^\top)^{-1}y \leq y^\top A^{-1}y\]
holds.
\end{lem}

\begin{proof}
By Sherman{-}Morrison formula \citep{Sherman1950} we have, 
\begin{align*}
y^\top(A+\alpha xx^\top)^{-1}y &= y^\top\left(A^{-1} - \frac{\alpha A^{-1}xx^\top A^{-1}}{1+\alpha x^TA^{-1}x}\right)y\\
&= y^\top A^{-1}y - y^\top\frac{\alpha A^{-1}xx^\top A^{-1}}{1+\alpha x^TA^{-1}x}y\\
&\leq y^\top A^{-1}y.
\end{align*}
The last inequality follows from the fact that $A^{-1}$ is positive definite.
\end{proof}

Using Lemma~\ref{lem:matrix-inequality}, we can prove Lemma~\ref{lem:bound-of-matrix-norm} as follows.

\begin{proof}[Proof of Lemma~\ref{lem:bound-of-matrix-norm}]
By the definition of $A_t$, we have
\[A_t = \lambda I + \sum_{k=1}^K T_k(t)x_kx_k^\top.\]
Then, for
\[\tilde A_t = \lambda I + \sum_{k=1}^K p^*_k(y(i,j))\ratioT{t}{i}{j}x_kx_k^\top,\]
we have
\[(y(i,j))^\top A^{-1}_t y(i,j) \leq (y(i,j))^\top \tilde A^{-1}_t y(i,j)\]
from Lemma~\ref{lem:matrix-inequality} and the fact
\[T_k(t) \leq p^*_k(y(i,j))\ratioT{t}{i}{j},\]
which can be inferred from the definition of $T_t(i,j)$. Therefore, proving 
\[(y(i,j))^\top \tilde A^{-1}_t y(i,j) \leq \frac{\rho(y(i,j))}{\ratioT{t}{i}{j}}\]
completes the proof of the lemma.

For convenience, we write $y(i,j)$ as $y$.
The KKT
condition of \eqref{eq:3} implies
that $w_i^*(y)$ and $p_i^*(y)$ satisfy
\begin{align*}
w^*_i(y) &= \frac12 p^*_i(y)x_i^\top\alpha\\
y &= \frac12 \sum_{i=1}^K p^*_i(y)x_ix_i^\top \alpha,
\end{align*}
where $\alpha \in \mathbb{R}^d$ corresponds to the Lagrange multiplier. Therefore, the optimal value $\rho(x)$ can be written as
\[\rho(y) = \sum_{i=1}^K \frac{{w^*}^2_i(y)}{p^*_i(y)} = \frac14\alpha^\top\left(\sum_{i=1}^K p^*_i(y) x_ix_i^\top\right) \alpha.\]
Now, let $B$ be denoted as  
\[B = \left(\sum_{i=1}^K p^*_i(y) x_ix_i^\top\right).\]
Then, since $y = \frac12 B\alpha$, we have
\begin{align*}
y^\top\tilde A^{-1}_t y - \frac{\rho(y)}{\ratioT{t}{i}{j}}&= \frac14 \alpha^\top B^\top \tilde A^{-1}_tB\alpha - \frac1{4\ratioT{t}{i}{j}} \alpha^\top B\alpha\\
&=\frac{1}{4}\alpha^\top \left(B^\top-\frac{\tilde A_t}{\ratioT{t}{i}{j}}\right) \tilde A^{-1}_tB\alpha\\
&=-\frac{1}{4}\alpha^\top \frac{\lambda}{\ratioT{t}{i}{j}} \tilde A^{-1}_tB\alpha\\
&\leq 0.
\end{align*}
The inequality follows from the fact that both of $\tilde A^{-1}_t$ and $B$ are positive semi-definite matrices. 
\end{proof}

\section{Proofs of Theorems}\label{sec:proof}
In this appendix,
we give the proofs of
Theorems~\ref{thm:justification} and \ref{thm:sample-complexity},
which are the main
theoretical contribution of this paper.
In the proof, we assume that the event $\mathcal{E}$ defined as
\[\mathcal{E} = \{\forall t>0,\,\forall i,j \in [K],\, |\Delta(i,j)-\hat{\Delta}_t(i,j)|\leq \beta_t(i,j)\}\]
occurs, where $\Delta(i,j)  = (x_i-x_j)^\top\theta$ is the gap of expected rewards between
arms $i$ and $j$.
The following lemma states that this assumption holds with high probability.
\begin{lem}
Event $\mathcal{E}$ holds w.p. at least  $1-\delta$.
\label{lem:event}
\end{lem}
Combining Prop.~\ref{prop:linear-confidence} and union bounds proves this lemma. 

\subsection{Proof of Theorem~\ref{thm:justification}}
Let $\tau$ be the stopping round of LinGapE. If $\Delta(a^*,\hat{a}^*) > \varepsilon$ holds, that is the returned arm $\hat{a}^*$ is worse than the best arm $a^*$ by $\varepsilon$, then we have
\begin{align*}
\Delta(a^*,\hat{a}^*)&> \varepsilon\geq B(\tau)\geq \hat{\Delta}_\tau(a^*,\hat{a}^*) + \beta_\tau(a^*,\hat{a}^*).
\end{align*}
The second inequality holds for stopping condition $B(\tau)\leq\varepsilon$ and the last follows from the definition of $B(\tau)$ (Line \ref{line:B-def} in Algorithm \ref{Select-Direction}). From this inequality, we can see that $\Delta(a^*,\hat{a}^*) > \varepsilon$ means that event $\mathcal{E}$ does not occur. Thus, the probability that LinGapE returns such arms is
\begin{align*}
\mathbb{P}[\Delta(a^*,\hat{a}_\tau) > \varepsilon] \leq \mathbb{P}[\bar{\mathcal{E}}] = 1 - \mathbb{P}[\mathcal{E}] \leq \delta,
\end{align*}
where $\bar{\mathcal{E}}$ represents the complement of the event $\mathcal{E}$. The last inequality follows from Lemma \ref{lem:event}. Therefore, we can conclude that the returned arm satisfies the condition \eqref{eq:stop}. 
\qed

\subsection{Proof of Theorem~\ref{thm:sample-complexity}}
We prove Theorem~\ref{thm:sample-complexity} by combining Lemma~\ref{lem:bound-of-matrix-norm} with following lemma.
\begin{lem}
Under event $\mathcal{E}$, $B(t)$ is bounded as follows. If $i_t$ or $j_t$ is the best arm, then
\[B(t) \leq \min(0,-(\Delta_{i_t}\vee \Delta_{j_t})+\beta_t(i_t,j_t)) +\beta_t(i_t,j_t). \]
Otherwise, we have
\[B(t) \leq \min(0,-(\Delta_{i_t}\vee \Delta_{j_t})+2\beta_t(i_t,j_t)) +\beta_t(i_t,j_t),\]
where
$a\vee b=\max(a,b)$.
\label{lem:bound-B}
\end{lem}

\begin{proof}
First, we consider the case where either arm $i_t$ or $j_t$ is
the best arm $a^*$. Since arm $i_t$ is the estimated best arm (Line~\ref{line:i-def} in Algorithm~\ref{Select-Direction}), we have 
\begin{align}
\hat{\Delta}_t(j_t,i_t) = (x_{j_t} -x_{i_t})^\top \theta_t^\lambda \leq 0.\label{eq:lem2-2}
\end{align}
Thus, $B(t)$ is bounded by
\begin{align}
B(t) = \hat{\Delta}(j_t,i_t) + \beta_{t}(i_t,j_t) \leq \beta_{t}(i_t,j_t). \label{eq:lem2-1}
\end{align}
Therefore, it is sufficient to show 
\begin{align}
B(t) \leq -(\Delta_{i_t}\vee\Delta_{j_t}) + 2\beta_t(i_t,j_t).\label{eq:lem2-target1}
\end{align}
If $i_t = a^*$, then 
\begin{align}
\left(\Delta_{i_t}\vee\Delta_{j_t}\right) = \Delta_{j_t} \label{eq:lem2-tmp1}
\end{align}
follows from the definition of $\Delta_a$ in \eqref{eq:delta-def}. In this case, $B(t)$ is bounded as 
\begin{align*}
B(t) &\overset{\text{(a)}}{=} \hat\Delta_t(j_t,i_t) + \beta_t(i_t,j_t)\\
&\overset{\text{(b)}}{\leq} \Delta(j_t,i_t) + 2\beta_t(i_t,j_t)\\
&\overset{\text{(c)}}{=} -\Delta_{j_t} + 2\beta_t(i_t,j_t)\\
&\overset{\text{(d)}}{=} -(\Delta_{i_t}\vee\Delta_{j_t}) + 2\beta_t(i_t,j_t),
\end{align*}
where (a), (b), (c) and (d) follow from the definition of $B(t)$, event $\mathcal{E}$, definition of $\Delta_a$ and \eqref{eq:delta-def}, respectively.

On the other hand, in the case where $j_t = a^*$, we have
\begin{align}
\left(\Delta_{i_t}\vee\Delta_{j_t}\right) = \Delta_{i_t} \label{eq:lem2-tmp2}.
\end{align}
In this case, the upper bound of $B(t)$ is derived as
\begin{align*}
B(t) &\overset{\text{(a)}}{\leq} \beta_t(i_t,j_t)\\
&\overset{\text{(b)}}{\leq} -\hat{\Delta}_t(j_t,i_t)+\beta_t(i_t,j_t)\\
&\overset{\text{(c)}}{\leq} -\Delta(j_t,i_t) + 2\beta_t(i_t,j_t)\\
&\overset{\text{(d)}}{=} -(\Delta_{i_t}\vee\Delta_{j_t}) + 2\beta_t(i_t,j_t),
\end{align*}
where (a), (b), (c) and (d) follow from \eqref{eq:lem2-1}, \eqref{eq:lem2-2}, event $\mathcal{E}$, and \eqref{eq:lem2-tmp2}, respectively.

Therefore, in both cases,
\eqref{eq:lem2-target1} holds, which completes the proof of the first inequality in Lemma~\ref{lem:bound-B}.

Next, we prove the second inequality, which holds when neither $i_t \neq a^*$ nor $j_t \neq a^*$. Again, with \eqref{eq:lem2-1}, it is sufficient to prove
\begin{align}
B(t) \leq -(\Delta_{i_t}\vee\Delta_{j_t}) + 3\beta_t(i_t,j_t). \label{eq:lem2-target2}
\end{align}

Since $j_t \neq a^*$,  
\begin{align}
\hat\Delta_t(a^*,i_t) + \beta_t(a^*,i_t) \leq \hat\Delta_t(j_t,i_t) + \beta_t(j_t,i_t). \label{eq:lem2-tmp3}
\end{align}
follows from the definition of $j_t$ (Line~\ref{line:j-def} in Algorithm~\ref{Select-Direction}). Thus, we have
\begin{align}
\beta_t(i_t,j_t) &\overset{\text{(a)}}{\geq} \hat\Delta_t(j_t,i_t) + \beta_t(j_t,i_t)\notag\\
&\overset{\text{(b)}}{\geq} \hat\Delta_t(a^*,i_t) + \beta_t(a^*,i_t)\notag\\
&\overset{\text{(c)}}{\geq} \Delta(a^*,i_t), \label{eq:lem2-3}
\end{align}
where (a), (b) and (c) follow from \eqref{eq:lem2-2}, \eqref{eq:lem2-tmp3}, event $\mathcal{E}$, respectively. By using \eqref{eq:lem2-3} and event $\mathcal{E}$,
we have
\begin{align}
B(t) &= \hat\Delta_t(j_t,i_t) + \beta_t(i_t,j_t)\notag\\
&\leq \Delta(j_t,i_t) + 2\beta_t(i_t,j_t)\notag\\
&= \Delta(j_t,a^*) + \Delta(a^*,i_t) + 2\beta_t(i_t,j_t)\notag\\
&\leq -\Delta_{j_t} + 3\beta_t(i_t,j_t).\label{eq:lem2-tmp4}
\end{align}
Moreover, from \eqref{eq:lem2-1} and \eqref{eq:lem2-3},
we obtain
\begin{align}
B(t) \leq 2\beta_t(i_t,j_t) \leq -\Delta_{i_t} + 3\beta_t(i_t,j_t).\label{eq:lem2-tmp5}
\end{align}
Combining \eqref{eq:lem2-tmp4} and \eqref{eq:lem2-tmp5} yields \eqref{eq:lem2-target2}, which was what we wanted.
\end{proof}

Based on Lemmas~\ref{lem:bound-of-matrix-norm} and \ref{lem:bound-B}, Theorem~\ref{thm:sample-complexity} is proved as follows.

\begin{proof}[Proof of Theorem~\ref{thm:sample-complexity}]
From Lemma~\ref{lem:event}, it suffices to show \eqref{eq:sample-complexity} and \eqref{eq:sample-complexity-high-dim} holds 
in the case where event $\mathcal{E}$ occurs.
First we derive the upper bound of $T_k(\tau)$. Let $\tilde t\leq \tau$ be the last round that arm $k$ is pulled. Then, 
\begin{align*}&\min(0,-\Delta_k+2\beta_{\tilde t-1}(i_{\tilde t-1},j_{\tilde t-1})) +\beta_{\tilde t-1}(i_{\tilde t-1},j_{\tilde t-1}))\\
&\qquad\geq B(\tilde t-1) \geq \varepsilon
\end{align*}
follows from Lemma~\ref{lem:bound-B} and the fact that stopping condition is not satisfied
at the $\tilde t$-th round.
Applying Lemma~\ref{lem:bound-of-matrix-norm} yields
\[\ratioT{\tilde t-1}{i_{\tilde t-1}}{j_{\tilde t-1}} \leq \frac{\rho(y(i_{\tilde t-1},j_{\tilde t-1}))}{\max\left(\varepsilon, \frac{\varepsilon+\Delta_{i_{\tilde t-1}}}{3},\frac{\varepsilon+\Delta_{j_{\tilde t-1}}}{3}\right)^2}C^2_{\tilde t-1},\]
where $C_t$ is defined in \eqref{eq:confidence-ellipsoid}.
Now, since arm $k$
is pulled at $\tilde t$-th round,
\[T_k(\tilde t-1) = p^*_k(y(i_{\tilde t-1},j_{\tilde t-1}))\ratioT{\tilde t-1}{i_{\tilde t-1}}{j_{\tilde t-1}}\] holds by definition. Therefore, $T_k(\tau)$ can be bounded as
\begin{align*}
T_k(\tau) &=  T_k(\tilde t-1) + 1\\
&=  p^*_k(y(i_{\tilde t-1},j_{\tilde t-1}))\ratioT{\tilde t-1}{i_{\tilde t-1}}{j_{\tilde t-1}} + 1\\
&\leq \max_{i,j\in [K]}  p^*_k(y(i,j))\ratioT{\tilde t-1}{i}{j} + 1\\
&\leq \frac{p^*_k(y(i_{\tilde t-1},j_{\tilde t-1}))\rho(y(i_{\tilde t-1},j_{\tilde t-1}))}{\max\left(\varepsilon, \frac{\varepsilon+\Delta_{i_{\tilde t-1}}}{3},\frac{\varepsilon+\Delta_{j_{\tilde t-1}}}{3}\right)^2}C^2_{\tilde t-1} + 1\\
&\leq \max_{i,j\in [K]} \frac{p^*_k(y(i,j))\rho(y(i,j))}{\max\left(\varepsilon, \frac{\varepsilon+\Delta_i}{3},\frac{\varepsilon+\Delta_j}{3}\right)^2}C^2_{\tau} + 1.
\end{align*}
Since $\sum_{k=1}^K T_k(\tau) = \tau$, summing up the upper bound of $T_k(t)$ above yields
\begin{align}
    \tau \leq H_\varepsilon C^2_\tau + K. \label{eq:key-for-sample-complexity}
\end{align}
Combined with
Lemmas~\ref{lem:complete-derivation} and \ref{lem:complete-derivation-2}
in Appendix~\ref{sec:proof-comp-sample}, we get what we wanted.
\end{proof}


\section{Lemmas for Proof of Theorem~\ref{thm:sample-complexity}}\label{sec:proof-comp-sample}
Here, we introduce two lemmas, which is used in the proof of Theorem~\ref{thm:sample-complexity} after having the inequality \eqref{eq:key-for-sample-complexity}.
Both lemmas are derived from
the following proposition given by \citet{Abbasi-Yadkori2010}.

\begin{prop}\label{prop:bound-of-confidence-interval}
\citep[Lemma 10]{Abbasi-Yadkori2010}
Let the maximum $l_2$ norm of features be denoted as $L$. Then, $\det(A^\lambda_n)$ is bounded as
\[\det(A^\lambda_n) \leq (\lambda + nL^2/d)^d.\]
\end{prop}

Now, we introduce the lemmas used in the final part of the proof of Theorem~\ref{thm:sample-complexity} as follows.

\begin{lem}
\label{lem:complete-derivation}
Let $C_t$ be
\[C_t = R\sqrt{2\log\frac{K^2\det(A_t)^{\frac12}\det(\lambda I)^{-\frac12}}{\delta}} + \lambda^{\frac12}S.\]
If
\begin{align}
\lambda \leq \frac{2R^2}{S^2}\log\frac{K^2}{\delta}, \label{eq:lambda-condition}
\end{align}
then 
\[\tau \leq H_\varepsilon C^2_\tau + K\]
implies
\begin{align}\label{eq:sample-complexity-2}
\tau \leq 8H_\varepsilon R^2\log\frac{K^2}{\delta} + C(H_\varepsilon,\delta),
\end{align}
where $C(H_\varepsilon,\delta)$ is
\[C(H_\varepsilon,\delta) = \mathcal{O}\left(H_\varepsilon\log\left(H_\varepsilon \log\frac{1}{\delta}\right)\right).\]
\end{lem}

\begin{proof}
From Proposition~\ref{prop:bound-of-confidence-interval}, we have 
\begin{align*}
C_\tau &\leq R\sqrt{2\log\frac{K^2}{\delta} + d\log\left(1+\frac{\tau L^2}{\lambda d}\right)}+\lambda^{\frac12}S \\
&\leq  2R\sqrt{2\log\frac{K^2}{\delta} + d\log\left(1+\frac{\tau L^2}{\lambda d}\right)}.
\end{align*}
The second inequality follows from condition \eqref{eq:lambda-condition}. Therefore, we can write
\begin{align*}
\tau &\leq H_\varepsilon C^2_\tau + K\\
&\leq 4H_\varepsilon R^2 \left(2\log\frac{K^2}{\delta} + d\log\left(1+\frac{\tau L^2}{\lambda d}\right)\right) + K.
\end{align*}

Let $\tau'$
a parameter satisfying
\begin{align}
\tau = 4H_\varepsilon R^2 \left(2\log\frac{K^2}{\delta} + d\log\left(1+\frac{\tau' L^2}{\lambda d}\right)\right) + K.\label{eq:tau-bound}
\end{align}
Then, $\tau' \leq \tau$ holds.

For $N$ defined as
\[N = 8H_\varepsilon R^2\log\frac{K^2}{\delta} + K,\]
we have
\begin{align*}
\tau' &\leq \tau\\
&= 4H_\varepsilon R^2d\log\left(1+\frac{\tau'L^2}{\lambda d}\right)+ N\\
&\leq 4H_\varepsilon R^2\sqrt{dL^2\tau'/\lambda} + N.
\end{align*}
By solving this inequality, we obtain
\begin{align*}
\sqrt{\tau'} &\leq  4H_\varepsilon R^2\sqrt{dL^2/\lambda} + \sqrt{16H^2_\varepsilon R^4dL^2\tau'/\lambda +N^2}\\
&\leq 2\sqrt{16H^2_\varepsilon R^4dL^2/\lambda +N^2}.
\end{align*}
Let $M$ be the right hand side of the inequality:
\[M = 2\sqrt{16H^2_\varepsilon R^4dL^2/\lambda +N^2}.\]
Then, using this upper bound of $\tau'$ in \eqref{eq:tau-bound} yields
\[\tau \leq 8H_\varepsilon R^2\log\frac{K^2}{\delta} + C(H_\varepsilon,\delta),\]
where $C(H_\varepsilon,\delta)$ is denoted as
\begin{align}
    C(H_\varepsilon,\delta) &= K + 4H_\varepsilon R^2d\log\left(1+\frac{M^2L^2}{\lambda d}\right)\label{eq:detail-C}\\ &=\mathcal{O}\left(H_\varepsilon\log\left(H_\varepsilon \log\frac{1}{\delta}\right)\right)\notag
\end{align}

\end{proof}

\begin{lem}
\label{lem:complete-derivation-2}
If $\lambda$ is set as $\lambda > 4H_\varepsilon R^2L^2$, then for $C_t$ defined as \eqref{eq:confidence-ellipsoid},
\[\tau \leq H_\varepsilon C^2_\tau + K\]
implies
\begin{align*}\tau \leq \left(8H_\varepsilon R^2\log\frac{K^2}{\delta}  + 2C'\right),
\end{align*}
where $C' = 2H_\varepsilon\lambda S^2+K$.
\end{lem}

\begin{proof}
Again, by Prop.~\ref{prop:bound-of-confidence-interval}, we have
\begin{align*}
C_\tau &\leq R\sqrt{2\log\frac{K^2}{\delta} + d\log\left(1+\frac{\tau L^2}{\lambda d}\right)}+\lambda^{\frac12}S.
\end{align*}
From the fact $(a+b)^2 \leq 2(a^2+b^2)$ and $(1+\frac{1}{x})^x \leq e$, we have 
\begin{align*}
\tau &\leq H_\varepsilon C^2_\tau + K\\
&\leq 2H_\varepsilon \left(2R^2\log\frac{K^2}{\delta} + \frac{\tau R^2L^2}{\lambda} + \lambda S^2\right) + K.
\end{align*}
Therefore, we can conclude that if $\lambda > 4H_\varepsilon R^2L^2$, then
\begin{align*}
\tau &\leq \left(1-\frac{2H_\varepsilon R^2L^2}{\lambda}\right)^{-1}\left(4H_\varepsilon R^2\log\frac{K^2}{\delta}  + C'\right)\\
&\leq 2\left(4H_\varepsilon R^2\log\frac{K^2}{\delta}  + C'\right).
\end{align*}
\end{proof}

\section{Proof of Theorem~\ref{thm:problem-complexity-bound}}\label{sec:proof-prob-complexity}
In this appendix
we give the proof of Theorem~\ref{thm:problem-complexity-bound}. This follows  straightforwardly from the definition of problem complexity $H_\varepsilon$ in \eqref{eq:problem-complexity} and the ratio $p^*_k(y(i,j)$ in \eqref{eq:p-def}.

\begin{proof}[Proof of Theorem \ref{thm:problem-complexity-bound}]
First, we bound the $\rho(y(i,j))$, which is the optimal value of
\begin{align}
\min_{p_k,w_k} \quad& \sum_{k=1} \frac{w_k^2}{p_k} \notag\\
\mathrm{s.t.} \quad&y(i,j) = \sum_{k=1}^K w_kx_k\notag\\
&\sum_{i=1}^K \,p_k = 1,\, p_k\geq 0,\,p_k, w_k \in \mathbb{R}.\label{eq:4}
\end{align}
Now, since $y(i,j) = y(i,a^*) + y(a^*,j)$,
$p'_k$ and $w'_k$ defined as
\begin{align*}
&p'_k = \frac{p_k^*(y(i,a^*))+p^*_k(y(a^*,j))}{2},\\
&w'_k = w_k^*(y(i,a^*)) + w_k^*(y(a^*,j))
\end{align*}
satisfy the condition of \eqref{eq:4}. Therefore, we have
\begin{align*}
\rho(y(i,j)) &\leq \sum_{k=1} \frac{(w'_k)^2}{p'_k}\\
&= 2\sum_{k=1} \frac{( w_k^*(y(i,a^*)) + w_k^*(y(a^*,j)))^2}{p_k^*(y(i,a^*))+p^*_k(y(a^*,j))}\\
&\leq 4\sum_{k=1} \frac{( w_k^*(y(i,a^*)))^2 + (w_k^*(y(a^*,j)))^2}{p_k^*(y(i,a^*))+p^*_k(y(a^*,j))}\\
&\leq 4\sum_{k=1} \frac{( w_k^*(y(i,a^*)))^2 }{p_k^*(y(i,a^*))}+\frac{(  (w_k^*(y(a^*,j)))^2}{p^*_k(y(a^*,j))}\\
&= 4\rho(y(i,a^*))+4\rho(y(a^*,j)).
\end{align*}
Using this upper bound, we can bound the problem complexity $H_0$ as follows. Let $i^*_k$ and $j^*_k$ be defined as
\[(i^*_k,j^*_k) = \argmax_{i,j\in[K]}\frac{p^*_k(y(i,j))\rho(y(i,j))}{\max\left(\Delta_i^2,\Delta_j^2\right)}.\]
and we have
\begin{align*}
H_0 &= 9\sum_{k=1}^K \max_{i,j\in[K]}\frac{p^*_k(y(i,j))\rho(y(i,j))}{\max\left(\Delta_i^2,\Delta_j^2\right)}\\
&= 9\sum_{k=1}^K \frac{p^*_k(y(i^*_k,j^*_k))\rho(y(i^*_k,j^*_k))}{\max\left(\Delta_{i^*_k}^2,\Delta_{j^*_k}^2\right)}\\
&\leq 36\sum_{k=1}^K p^*_k(y(i^*_k,j^*_k)) \frac{\rho(y(i^*_k,a^*))+\rho(a^*,j^*_k))}{\max\left(\Delta_{i^*_k}^2,\Delta_{j^*_k}^2\right)}\\
&\leq 36\sum_{k=1}^K p^*_k(y(i^*_k,j^*_k)) \left(\frac{\rho(y(i^*_k,a^*))}{\Delta_{i^*_k}^2}+\frac{\rho(a^*,j^*_k))}{\Delta_{j^*_k}^2}\right)\\
&\leq 72H'_{\text{oracle}}.
\end{align*}
which was what we wanted. The last inequality holds from $\sum_{k=1}^K p^*_k(i,j) = 1$ for all $i,j\in[K]$.
\end{proof}

\end{document}